\documentclass{article}

\usepackage{times}
\usepackage{graphicx} 
\usepackage{subfigure} 

\usepackage{natbib}

\usepackage{algorithm}
\usepackage{algorithmic}

\usepackage{hyperref}

\usepackage[accepted]{icml2017}

\usepackage{multirow}
\usepackage{amssymb, dsfont}

\usepackage{amsmath, amssymb, bm, amsthm}
\usepackage{mathtools}
\usepackage{thmtools}
\usepackage{enumitem}
\usepackage{subfigure}
\usepackage{color}

\usepackage[capitalize,nameinlink]{cleveref}

\newtheorem{thm}{Theorem}[section]
\newtheorem{lem}{Lemma}[section]

\def\[#1\]{\begin{align}#1\end{align}}

\newcommand{\dee}{\mathrm{d}}
\newcommand{\grad}{\nabla}

\DeclareMathOperator{\Bernoulli}{Bernoulli}
\DeclareMathOperator{\gammadist}{gamma}
\DeclareMathOperator{\betadist}{beta}
\DeclareMathOperator{\Dirichlet}{Dirichlet}
\DeclareMathOperator{\categorical}{categorical}

\DeclareMathOperator{\uniform}{Uniform}

\DeclareMathOperator{\shp}{shp}
\DeclareMathOperator{\rte}{rte}

\DeclareMathOperator{\GRGLAW}{GRG}

\newcommand{\EE}{\mathbb{E}}
\newcommand{\Lcal}{\mathcal{L}}

\newcommand{\card}[1]{\vert {#1} \vert}

\renewcommand{\Pr}{\mathbb{P}}
\newcommand{\given}{\mid}
\newcommand{\dist}{\ \sim\ }

\newcommand{\Indicator}[1]{\mathds{1}_{\{#1\}}}

\newcommand{\defas}{\vcentcolon=}  
\newcommand{\iid}{i.i.d.}

\newcommand{\discount}{\alpha}
\newcommand{\ubound}[1]{C_{#1}}

\newcommand{\conc}{c}

\newcommand{\edges}{\mathcal E}
\newcommand{\minibatch}{\mathcal B}
\newcommand{\gammacdf}{F}
\newcommand{\igammacdf}{\gammacdf^{-1}}
\newcommand{\totalmass}{\gamma}
\newcommand{\degree}[1]{D_{n,#1}}

\newcommand{\pto}{\stackrel{\mathbb P}{\to}}
\newcommand{\ratio}{M}
\newcommand{\pgf}[1]{F_{#1}}
\newcommand{\mpgf}[2]{F_{#1}^{(#2)}}
\newcommand{\ingf}[1]{Q_{#1}}

\definecolor{WowColor}{rgb}{.75,0,.75}
\definecolor{SubtleColor}{rgb}{0,0,.50}

\newcommand{\PROBLEM}[1]{\textcolor{WowColor}{ {\bf (!!)} {\bf #1}}}

\newcounter{margincounter}

\usepackage{booktabs}
\usepackage{array}
\newcolumntype{M}[1]{>{\centering\arraybackslash}m{#1}}

\icmltitlerunning{Power law simple graphs}

\begin{document} 

\twocolumn[
\icmltitle{Bayesian inference on random simple graphs\\
		with power law degree distributions}




\begin{icmlauthorlist}
\icmlauthor{Juho Lee}{pos}
\icmlauthor{Creighton Heaukulani}{cam}
\icmlauthor{Zoubin Ghahramani}{cam,ub}
\icmlauthor{Lancelot F. James}{ust}
\icmlauthor{Seungjin Choi}{pos}
\end{icmlauthorlist}

\icmlaffiliation{pos}{Pohang University of Science and Technology, Pohang, South Korea}
\icmlaffiliation{cam}{University of Cambridge, Cambridge, UK}
\icmlaffiliation{ub}{Uber AI Labs, San Francisco, CA, USA}
\icmlaffiliation{ust}{Hong Kong University of Science and Technology, Hong Kong}

\icmlcorrespondingauthor{Juho Lee}{stonecold@postech.ac.kr}
\icmlcorrespondingauthor{Seungjin Choi}{seungjin@postech.ac.kr}

\icmlkeywords{relational data, network models, scale-free random graphs, variational inference}

\vskip 0.3in
]



\printAffiliationsAndNotice{}  

\begin{abstract} 
We present a model for random simple graphs with power law (i.e., heavy-tailed) degree distributions. 
To attain this behavior, the edge probabilities in the graph are constructed from Bertoin--Fujita--Roynette--Yor (BFRY) random variables, 
which have been recently utilized in Bayesian statistics for the construction of power law models in several applications.
Our construction readily extends to capture the structure of latent factors, similarly to stochastic blockmodels, while maintaining its power law degree distribution.
The BFRY random variables are well approximated by gamma random variables in a variational Bayesian inference routine, which we apply to several network datasets for which power law degree distributions are a natural assumption.
By learning the parameters of the BFRY distribution via probabilistic inference, we are able to automatically select the appropriate power law behavior from the data.
In order to further scale our inference procedure, we adopt stochastic gradient ascent routines where the gradients are computed on minibatches (i.e., subsets) of the edges in the graph.
\end{abstract} 

\section{Introduction}
\label{sec:intro}

In statistical applications, random graphs serve as Bayesian models for network data, that is, data consisting of objects and the observed linkages between them.
Here we will focus on models for random \emph{simple graphs} (that is, graphs with edges that take binary values), which are appropriate for applications where we observe either the presence or absence of links between objects in the network.
For example, in social networks, nodes may represent individuals and a link (i.e., a nonzero value of an edge) could represent friendship. 
In a protein-protein interaction network, nodes may represent proteins and links could represent an observed physical or chemical interaction between proteins.
Many domains involving network data (including social and protein-protein interaction networks) have been shown to exhibit power law, i.e., heavy-tailed, degree distributions \citep{barabasi1999emergence}. 
Models for random graphs with power law degree distributions, also called \emph{scale-free} random graphs, have therefore become one of the most actively studied areas of graph theory and network science \citep{bollobas2001degree,albert2002statistical,dorogovtsev2002evolution}. 
In this paper we present a model for simple, scale-free random graphs, which we apply as a probabilistic model for several network datasets.

The model we present in this paper is a special case of the \emph{generalized random graph} defined by \citet{britton2006generating}, and studied further by \citet[Ch.~6]{van2016random}, which outlines a framework for defining scale-free random graphs, but does not provide practical constructions, much less algorithms for performing statistical inference on the model components given data. 
Here we provide one such practical construction, along with a \emph{variational inference} routine \citep{jordan1999introduction} for efficient posterior inference. 
What's more, our construction readily generalizes to include the structure of latent factors/clusters, as captured by the popular \emph{stochastic blockmodels} \citep{nowicki2001estimation,airoldi2009mixed}, while maintaining power law behavior in the graph.
%

Applying Bayesian inference algorithms on network datasets is a challenge because likelihood computations, in general, scale with the number of edges in the graph, which is $O(n^2)$ in a network with $n$ nodes.  
To help overcome these difficulties, we follow \citet{hoffman2013stochastic} 
and develop a \emph{stochastic variational inference} algorithm in which we approximate many likelihood computations on only subsets of the data, called \emph{minibatches}.
In the case of a network dataset, the minibatches are comprised of subsets of edges in the graph. 
%
%

We apply this inference procedure to several network datasets that are commonly observed to possess power law structure. Our experiments show that accurately capturing this power law structure improves performance on tasks predicting missing edges in the networks. 

%
\enlargethispage{\baselineskip}

\section{Bayesian models for simple graphs}
\label{sec:simplegraphs}

We represent a simple graph with $n$ nodes by an adjacency matrix $X \defas (X_{i,j})_{i,j\le n}$, where $X_{i,j} = 1$ if there is a link between nodes $i$ and $j$ and $X_{i,j}=0$ otherwise.
Here we will only consider undirected graphs, in which case $X$ represents a symmetric matrix.
 Furthermore, we do not allow self links, so the diagonal entries in $X$ are meaningless.
Most probabilistic models for simple graphs take the entries in $X$ to be conditionally independent Bernoulli random variables; in particular, for every $i, j \le n$, let $p_{i,j}$ be the (random) probability of a link between nodes $i$ and $j$, and let $X_{i,j} \given p_{i,j} \dist \Bernoulli( p_{i,j} )$.
For every simple graph $x \defas (x_{i,j})_{i,j\le n}$, we may then write the likelihood for the parameters $p \defas (p_{i,j})_{i,j\ge 1}$ given $X$ as
\[
\label{eq:likel1}
P( X = x \given p ) 
	= \prod_{i < j \le n} p_{i,j}^{x_{i,j}} 
		(1-p_{i,j})^{1-x_{i,j}}
		,
\]
where in our case it should be clear that the product is only over $i,j \le n$ such that $i<j$ and $i\ne j$.
Random simple graphs date back to the Erd{\"o}s--R{\'e}nyi model, which may be reviewed, along with the more general theory of random graphs, in the text by \citet{bollobas1998random}.
A random graph is called scale-free when the fraction of nodes in the network having $k$ connections to other nodes behaves like $k^{-\tau}$ for large values of $k$ and some exponent $\tau>1$.
More precisely, let $\degree{i} \defas \sum_{j \ne i} X_{i,j}$ denote the (random) degree of node $i$, for every $i\le n$.
Then $X$ is (asymptotically) scale-free when, for every node $i\le n$,
\[
\label{eq:powerlawdef}
\Pr\{ \degree{i} = k \} \sim c k^{-\tau}
	,
	\quad
	\text{ as } n\rightarrow\infty
	,
\]
for some constant $c>0$, a power law exponent $\tau>1$, and $k$ sufficiently large. 
Here the notation $A \sim B$ denotes that the ratio $A/B \rightarrow 1$ in the specified limit.

In order to model scale-free random graphs, \citet{britton2006generating} suggested reparameterizing the model in \cref{eq:likel1} by a sequence of odds ratios $r_{i,j} \defas p_{i,j} / (1-p_{i,j})$, for every $i < j\le n$, which factorize as $r_{i,j} = U_i U_j$, for some $U \defas (U_1, \dotsc, U_n)$.
The node-specific factors $U_i$ are then modeled as $U_i \defas W_i / \sqrt{L}$ for some sequence of nonnegative random variables $W \defas (W_1, \dotsc, W_n)$ and where $L\defas \sum_{i=1}^n W_i$. 
In a series of results, \citep[Thms.~3.1 \& 3.2]{britton2006generating} and \citep[Cor.~6.11 \& Thm.~6.13]{van2016random} assert conditions on the random variables $W$ so that the limiting distribution of the degrees $\degree{i}$ is a mixed Poisson distribution.
We will further detail these previous results in \cref{sec:relatedwork}. \vfill\null
The distribution of $W_i$ is interpreted here as a prior distribution for the degree $\degree{i}$ of node $i$, and if its distribution has heavy tails, then so will the distribution of $\degree{i}$. 
Conversely, if the distribution of $W_i$ does not have heavy tails, then neither will the distribution of the degrees $\degree{i}$.  We explore this alternative in \cref{sec:experiments}.

Previous authors did not suggest any particular choices for the distribution of $W_i$, and so we elect to model them with BFRY random variables \citep{bertoin2006particular,devroye2014simulation}, which have a heavy-tailed distribution and have recently played a role in the construction of several power law models in Bayesian statistics.
Other heavy tailed distributions, such as those exhibited by log normal random variables, may also be used to model the $W_i$, and these options may be explored. 
One benefit of the BFRY distribution is that the thickness of its tails, and thus the power law behavior of the resulting graph, may be straightforwardly controlled by the discount parameter $\discount$.


\section{A generalized random graph}
\label{sec:model}

Consider the model from the previous section, parameterized by the odds ratios $r\defas (r_{i,j} \colon i<j\le n)$.
Define
\[
G(r) \defas \prod_{i<j\le n} (1+r_{i,j})
	= \prod_{i<j\le n} (1+U_i U_j)
	,
	\label{eq:likel_denom}
\]
and note that the conditional likelihood in \cref{eq:likel1} may be rewritten in terms of the degrees $\degree{i}$ as 
\[
P ( X = x \given r )
	&= G(r)^{-1} \prod_{i<j \le n} (U_i U_j)^{x_{i,j}}
	\\
	&= G(r)^{-1} \prod_{i\le n} U_i^{\degree{i}}
	.
	\label{eq:likel}
\]
%
The random simple graph $X$ is called a \emph{generalized random graph}, and we will henceforth write $X \given r \dist \GRGLAW( n, r )$.
Let $\discount \in (0,1)$, which we call the \emph{discount parameter}, and let $\ubound{1}, \ubound{2}, \dotsc$ be a sequence of positive values satisfying
\[
\lim_{n\rightarrow \infty} 
	\ubound{n} = \infty
	\quad
	\text{and}
	\quad
	\lim_{n\rightarrow \infty} 
	\ubound{n}^\discount / n = 0
	.
	\label{eq:creq}
\]
Let the weights $W_1, \dotsc, W_n$ be \iid\ with density
\[
f_n(w) \propto w^{-\discount-1}
	(1-e^{-w}) \Indicator{ 0 \le w \le \ubound{n} }
	.
	\label{eq:densityw}
\] 
(These are truncated BFRY random variables and will be discussed, along with a method for simulation, in \cref{sec:wrvs}.) 
Then the corresponding generalized random graph has an (asymptotic) power law degree distribution with power law exponent $\tau = 1+\discount$. 
We summarize this construction in the following theorem:
%
%
\begin{thm}
\label{result:main}
For every $n$, let $W_1, \dotsc, W_n$ be \iid\ with density $f_n$ and let $(\degree{i})_{i\le n}$ be the degrees of the generalized random graph $X \given r \dist \GRGLAW(n, r)$, where $r \defas (r_{i,j})_{i<j\le n}$ is the sequence of odds ratios
\[
r_{i,j} = W_i W_j / L 
	,
	\qquad i<j \le n
	,
\]
and $L\defas \sum_i W_i$.
Then the following hold:
\begin{enumerate}[nolistsep]
\item For $y\gg 1$, $\Pr\{ \degree{i}=y\} \sim c y^{-1-\alpha}$, for every node $i$ and for some constant $c$, as $n\to \infty$.


\item For any $m$, the collection $\degree{1}, \dotsc, \degree{m}$ are asymptotically independent, as $n\to \infty$.
\end{enumerate}
\end{thm}
%
This construction is closely related to the model described by \citet[Thm.~6.13]{van2016random}, and the proof of \cref{result:main}, which is provided in the supplementary material, follows analogously to the results by \citet[{Thms.~3.1 \& 3.2}]{britton2006generating}.
Note that the power law exponent $\tau = 1+\discount$ of the graph (as described by \cref{eq:powerlawdef}) is determined by the parameter $\discount \in (0,1)$, and takes values in $(1,2)$. While power law exponents in $(2,3)$ has often been suggested in the past, it has more recently been shown that exponents within the $(1,2)$ range of our model is more appropriate in many domains \citep[Ch.~1]{van2016random}; \citep{crane2015atypical}.

\begin{figure}[t]
\centering
	\includegraphics[scale=0.38]{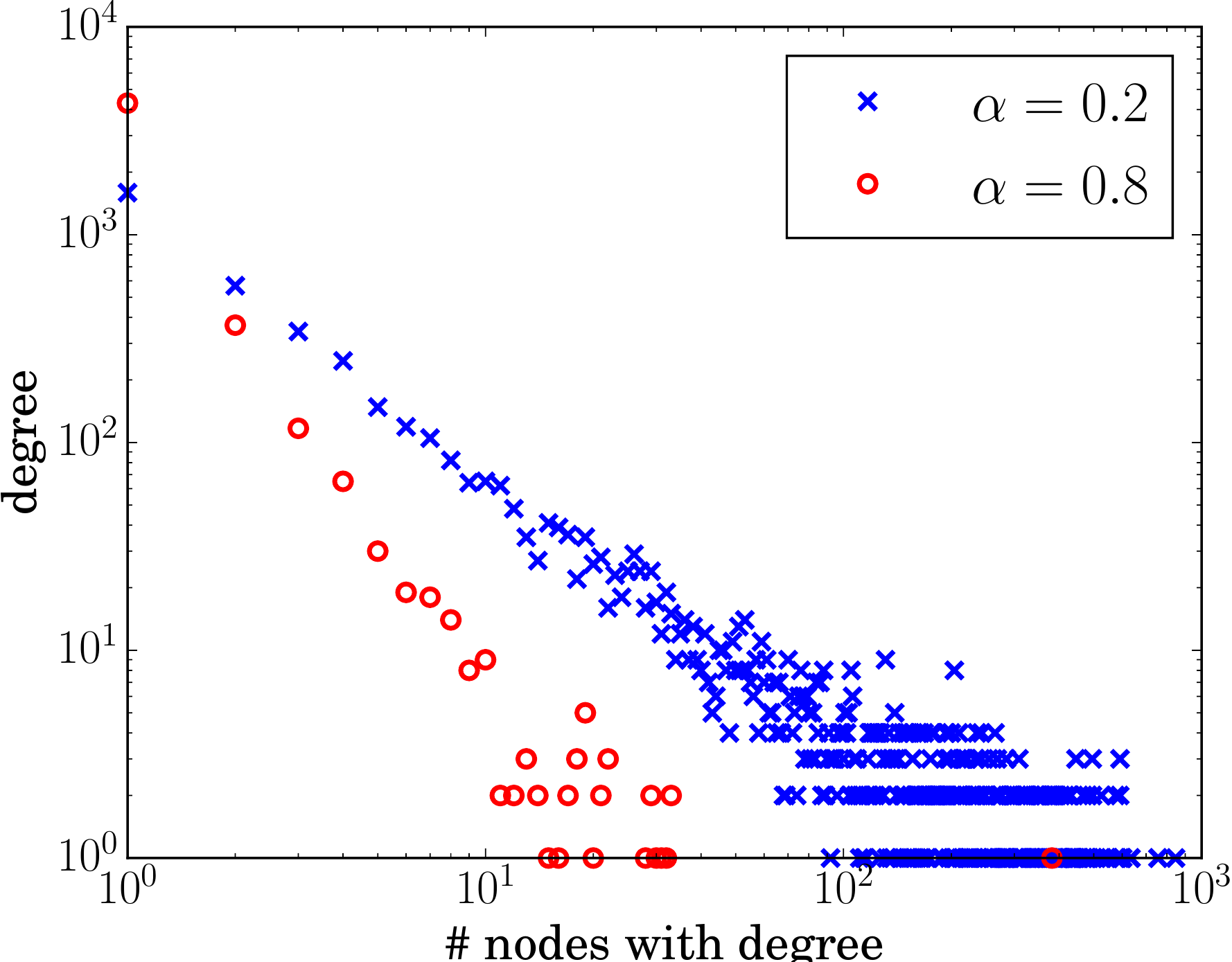}
	\label{fig:powerlaw}
	\caption{The number of nodes with various degrees for two simulated graphs with $n=3000$ nodes and differing values for $\discount$.}
\end{figure}

\subsection{Truncated BFRY random variables}
\label{sec:wrvs}

A random variable $W$ with density function $f_n$ given by \cref{eq:densityw} is a ratio of gamma and beta random variables, upper truncated at $\ubound{n}$.
In particular let
\[
g \dist \gammadist(1-\discount, 1)
	\quad
	\text{and}
	\quad
	b \dist \betadist(\discount, 1)
	,
	\label{eq:gammabeta}
\]
be independent, 
then the ratio $Z \defas g/b$ has density $p(z) \propto z^{-\discount-1} (1-e^{-z})$ on $(0,\infty)$ (by construction), which is known as the Bertoin-Fujita-Roynette-Yor (BFRY) distribution \citep{bertoin2006particular,devroye2014simulation} and has been used in the construction of power law models in some recent applications in machine learning \citep{james2015scaled,lee2016finite}.
The random variable $W$ is then obtained by upper truncating the random variable $Z$ at $\ubound{n}$.
By our requirements on the sequence $\ubound{n}$ (c.f.~\cref{eq:creq}), the density function $f_n$ of $W$ approaches the density function of the BFRY random variable $Z$ as $n\rightarrow \infty$, that is,
\[
\lim_{n\rightarrow \infty} f_n(w) 
	= \frac{\discount}{\Gamma(1-\discount)} w^{-\discount-1}
		(1-e^{-w})
	,
\]
which is heavy-tailed with infinite moments.
%
It is straightforward to simulate these truncated BFRY random variables by repeatedly simulating $g$ and $b$ as in \cref{eq:gammabeta}, accepting $W \defas g/b$ as a sample when $W < \ubound{n}$.

The truncation of $W$ at $\ubound{n}$ produces a random variable with finite mean (for $n<\infty$), which is essential when constructing the generalized random graph and motivates the construction by \citet[Thm.~6.13]{van2016random} alluded to earlier; see \cref{sec:relatedwork}.
For simplicity, one could take $\ubound{n} = n$, but the flexibility to set this parameter allows us to control other properties of the model.  For example, in the next section we show how to vary this truncation level to control the sparsity of the graph.  

\subsection{Controlling power law and sparsity in the graph}
\label{sec:sparsity}


\begin{figure}
	\centering
	\includegraphics[scale=0.38]{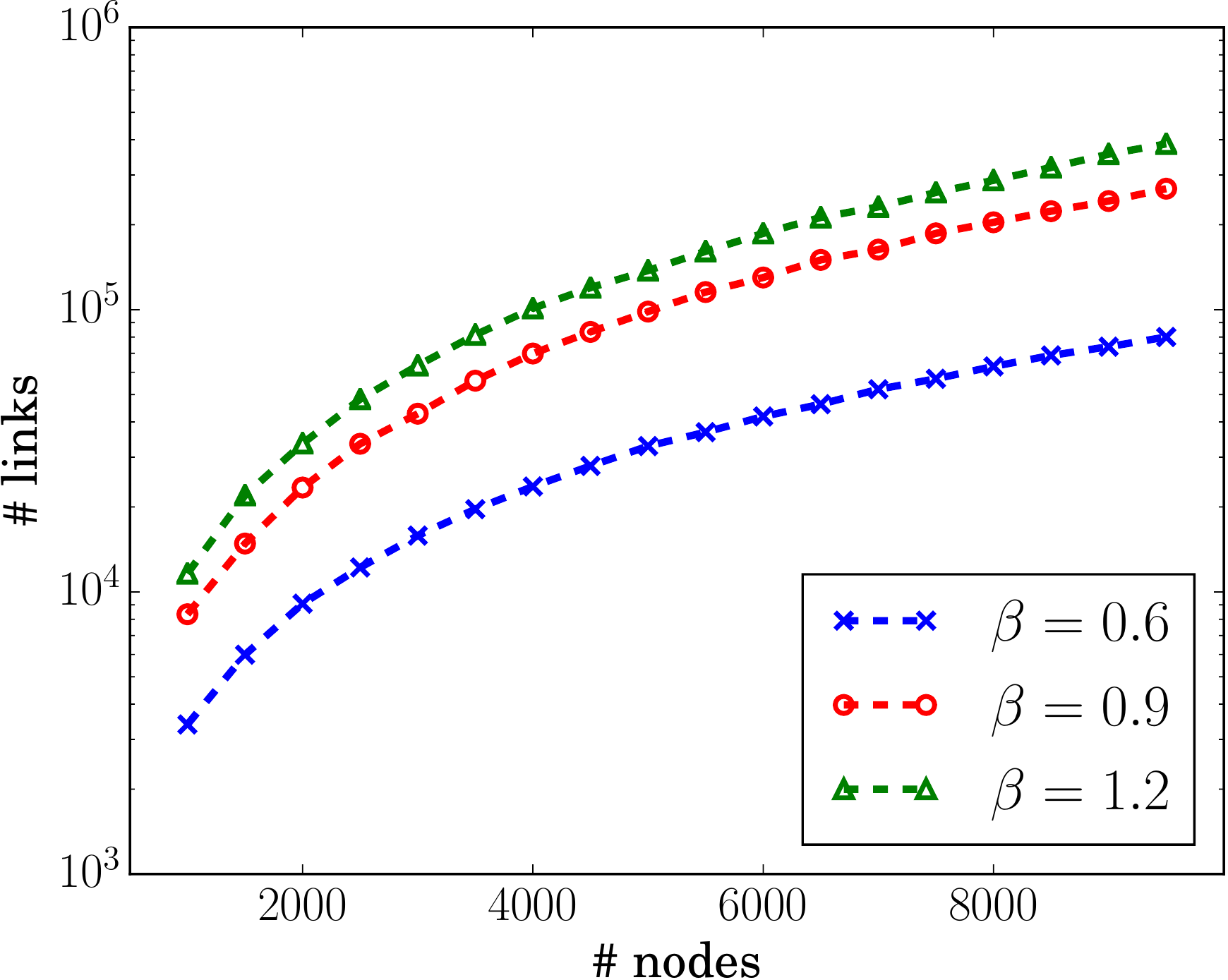}
	\label{fig:sparsity}
	\caption{The average number of links in simulated graphs with varying sparsity parameter $\beta$.}
\end{figure}


The discount parameter $\discount \in (0,1)$ controls the power law behavior of the graph, where decreasing $\discount$ results in heavier tails in the degree distribution of the nodes in the graph.
We can visualize this behavior by simulating graphs at different values of $\discount$.
In \cref{fig:powerlaw}, we set $\ubound{n}=n$ and show the number of nodes of varying degrees in two simulated graphs, one with $\discount = 0.2$ and one with $\discount=0.8$.

The degree distribution of the nodes in a graph of course affects the sparsity of the graph; to characterize this relationship, we can upper bound the expected number of links in the graph as follows:
\begin{thm}
\label{result:sparsity}
Let $E_n$ be the number of positive edges in the graph.  Then
$
\EE [E_n] = O(n \ubound{n}^{1-\discount}).
$
\end{thm}
The derivation of this result is provided in the supplementary material.
%
%
While varying $\discount$ can thus control the sparsity of the graph in addition to the power law behavior, we often want to decouple these behaviors, in which case we could parameterize the truncation level as $\ubound{n} = n^{\beta}$, for some \emph{sparsity parameter} $\beta > 0$.  Note the restriction $\alpha<\min\{1,1/\beta\}$ must be enforced in order to ensure that the conditions in \cref{eq:creq} are satisfied.
In this case, the bound in \cref{result:sparsity} becomes $\EE [E_n] = O(n^{1+\beta(1-\discount)})$. 
The interpretation here is that increasing the upper bound $\ubound{n}$ increases the likelihood that any particular node will link to others, but does not affect the (asymptotic) power law characterized by \cref{result:main}.
In \cref{fig:sparsity}, we display the average number of positive edges in graphs that were simulated with fixed $\discount=0.3$ and varying values of the sparsity parameter $\beta$. 
We note that in simulations, we encountered numerical issues in $\beta>1.4$ regimes.

\section{Related work}
\label{sec:relatedwork}

%

Referring to the construction for generalized random graphs in \cref{sec:simplegraphs}, \citet[Thm.~3.1]{britton2006generating} shows that when the weights $W_i$ have finite first and second moments, then the limiting distribution of the degree $\degree{i}$ is a mixed Poisson distribution. 
Most such distributions are light-tailed, however, in which case the degrees will not exhibit power law behavior.
\citet[Thm.~3.2]{britton2006generating} therefore provides an alternative construction in which $W_i$ may have infinite moments (so that it may exhibit a heavy tail), which results in a graph with a power law exponent of $\tau=2$. 
Finally, \citet[Thm.~6.13]{van2016random} suggests yet another construction where the $W_i$ are upper truncated to be of order $o(n)$, where $n$ is the number of nodes in the graph. 
The resulting random variables therefore have finite moments, yet exhibit a heavy tail, and the resulting random graph has a heavy tailed degree distribution with an arbitrary power law exponent.
None of these results suggest a particular choice for the distribution of $W_i$, however, and so we have elected to use BFRY random variables (which are heavy tailed) that are upper truncated (so that they have finite moments).  We note that the requirements on our truncation level (c.f.~\cref{eq:creq}) is less strict than the $o(n)$ criterion of the \citet[Thm.~6.13]{van2016random} construction.

The reader may consult the surveys by \citet{bollobas2003mathematical,albert2002statistical,dorogovtsev2002evolution} for a background on scale-free random graphs, which is too large to review here.
%
%
While these models are numerous, the following recent pieces of work in the Bayesian statistics and machine learning communities may be of interest to the reader: \citet{caron2014sparse,veitch2015class,crane2016edge,cai2015completely}. 
This collection of work discusses power law degree distributions, albeit in some cases in multi-graphs (i.e., graphs with nonnegative integer-valued edges) and in some cases the power law behavior is not characterized, only numerically observed in simulations. 
Many of these models can be seen to invoke their power law properties from the \emph{Pitman--Yor process} \citep{pitman1997two} (or related stochastic processes), where the extent of this behavior is controlled by the discount parameter $\discount \in (0,1)$ of the Pitman--Yor model, which, like the BFRY distribution, is related to a stable subordinator of index $\discount$.

\section{Incorporating latent factors}
\label{sec:blockmodels}

Latent factor models for relational data assume that a set of latent clusters underlie the network.
For example, in a social network, the latent factors could be the unobserved hobbies or interests of individuals, which determine the observed friendships in the network.
Bayesian models for latent factors in relational data are widespread, with some of the most popular based on \emph{stochastic blockmodels}, where models for unsupervised learning, or clustering, are used to infer the latent factors \citep{nowicki2001estimation,kemp2006learning,airoldi2009mixed,miller2009nonparametric}.
%
%
In this section, we present extensions of the generalized random graph that incorporate latent factors by scaling the odds ratios, while maintaining their power law degree distribution.

We will first provide a general result showing how to incorporate random scaling variables into the model, followed by specific examples that model these scaling variables with latent clusters.
Let the odds ratios in the generalized random graph be given by $r_{i,j} = A_{i,j} U_i U_j$ for some $A_{i,j} \geq 0$.
Note that $p_{i,j} \to 1$ as $A_{i,j} \to \infty$ and $p_{i,j} \to 0$ as $A_{i,j} \to 0$, and so the edge-specific weight $A_{i,j}$ simply scales the link probability.
The random graph $X \given r \dist \GRGLAW(n, r)$ then has the likelihood
\[
P ( X = x \given r ) 
	= 
	G(r)^{-1}
	\prod_{i<j\le n} A_{i,j}^{x_{i,j}}
	\prod_{i\le n} U_i ^{\degree{i}}
	,
\]
where the normalization term $G(r)$ in \cref{eq:likel_denom} is now
\[
G(r) &\defas \prod_{i < j \le n} (1 + A_{i,j} U_i U_j)
	\\
	&\: = \sum_x \prod_{i < j \le n} A_{i,j}^{x_{i,j}} \prod_{i\le n} U_i^{\degree{i}}
	,
\]
where the final equality follows simply because $\sum_x P(X=x \given r) = 1$.
%
%
So constructed, the odds ratios $r$ will influence the link probabilities in the generalized random graph, but will not affect the power law behavior of the degree distributions (under some assumptions on the random variables $A_{i,j}$).
We summarize this construction in the following theorem, the proof for which is provided in the supplementary material:
\vfill\null

\begin{thm}
\label{result:grg_scaled_power_law}
Let $(W_i)_{i\le n}$ be \iid\ random variables with density function $f_n(w)$ (in \cref{eq:densityw}).
Let $(A_{i,j})_{i < j \leq n}$ be a collection of uniformly bounded random variables, 
where, for every $i\le n$, the collection $(A_{i,j})_{j>i}$ is exchangeable.
Let $(\degree{i})_{i\le n}$ be the degrees of the random graph $X \given r \sim \GRGLAW(n, r)$, where $r \defas (r_{i,j})_{i<j\le n}$ is the sequence of odds ratios
\[
r_{i,j} = A_{i,j} W_i W_j / L
	,
	\qquad
	i<j\le n
	,
\]
and where $L \defas \sum_i W_i$.
Then the degrees $(\degree{i})_{i\le n}$ satisfy statements (1) and (2) in \cref{result:main}.
\end{thm}

For example, we may construct \emph{stochastic blockmodels}, such as those introduced by \citet{nowicki2001estimation}, as follows:
For every $i\le n$, let $Z_i$ be a random variable taking values in $\{1,\dotsc,K\}$, indicating which one (and only one) of $K$ different factors to associate with node $i$.
We want the latent cluster assignments for two nodes $i$ and $j$ to influence their link probability, which we could capture with a set of parameters $\theta_{k,\ell}$, for $k, \ell = 1,\dotsc, K$.
Then the parameter $\theta_{Z_i, Z_j}$ could represent, or influence, the probability of a link between nodes $i$ and $j$.
Taking a Bayesian approach, the indicator variables $Z_i$ may be modeled with a Dirichlet-categorical conjugate distribution and their values may be inferred via probabilistic inference. 
An example of such a model could be summarized as follows: Let
\begin{alignat}{2}
Z_i &\dist \categorical(\pi) ,
	&&\qquad i\le n
	,
	\\
	\pi &\dist \Dirichlet(\conc/K)
	,
	&&\qquad \text{where } \conc>0
	,
	\\
	\theta_{\ell,k} &\dist \gammadist ( a_\theta, b_\theta )
		,
		&&\qquad \ell, k \le K
		,
	\\
	A_{i,j} &\,\, = \,\, \theta_{Z_i, Z_j}
	,
	&&\qquad i < j \le n
	,
\end{alignat}
and construct the random graph $X$ as in \cref{result:grg_scaled_power_law}.
\citet{kemp2006learning} developed a nonparametric extension of a similar model that in a sense takes the limit $K\rightarrow \infty$, allowing an appropriate number of clusters to be automatically inferred from the data.  In this case, the marginal law of the indicator variables $Z_1, \dotsc, Z_n$ is given by a Chinese restaurant process (with concentration parameter $\conc$).

Several generalizations of the stochastic blockmodel allow the clusters underlying the network to overlap, leading to \emph{mixed membership stochastic blockmodels} \citep{airoldi2009mixed} or the related \emph{latent feature relational models} \citep{miller2009nonparametric}.
To capture this structure, we may generalize the indicators $Z_i$ to now represent a binary $K$-vector with entry $Z_{i,k}=1$ indicating node $i$ is associated with cluster $k$, now called a \emph{feature}, and $Z_{i,k}=0$ otherwise.  One example of such a model could be summarized as follows:
\begin{alignat}{2}
Z_{i,k} &\dist \Bernoulli ( p_k )
	,
	&&\quad i\le n , k\le K,
	\\
	p_k &\dist \betadist( \conc, \conc \totalmass/K )
	,
	&&\quad k\le K , \text{and } \conc, \totalmass>0
	,
	\\
	\theta_{\ell,k} &\dist \gammadist ( a_\theta, b_\theta )
		,
		&&\quad \ell, k = 1, 2, \dotsc
		,
	\\
	A_{i,j} &\,\, = \,\, \sum_{k,\ell} \theta_{k,\ell} Z_{i,k} Z_{j,\ell}
	,
	&&\quad i < j \le n
	,
\end{alignat}
and construct the random graph $X$ as in \cref{result:grg_scaled_power_law}.
\citet{miller2009nonparametric} derived a nonparametric extension of this model that in a sense takes the limit $K\rightarrow \infty$, in which case the marginal law of the vectors $Z_1, \dotsc, Z_n$ is that of an \emph{Indian buffet process} (with mass parameter $\totalmass$ and concentration parameter $\conc$) \citep{GGS2007}.

\section{Variational inference}
\label{sec:inference}

We derive a variational Bayesian inference algorithm \citep{jordan1999introduction} that approximates the (optimal state of the) posterior distribution of the model components, given a network dataset.
We approximate the required gradients in this procedure with stochastic gradient ascent \citep{bottou2010large,hoffman2013stochastic}, computed on minibatches (i.e., subsets) of edges in the graph.

\subsection{The variational lower bound}

In variational inference, we approximate the posterior distribution on the latent variables $W \defas (W_1, \dotsc, W_n)$ with a variational distribution $q(W; \theta)$, the parameters $\theta$ of which are fit to maximize the following lower bound on the marginal likelihood
\[
\log p(X) 
	&\ge \EE_{q(W; \theta)} \Bigl [
	\log \frac{p(X \given W ; \discount) p(W; \discount)}{q(W; \theta)} 
	\Bigr ]
	,
	\label{eq:elbo1}
\]
where 
$p(X \given W)$ is the likelihood function computed as in \cref{eq:likel}, and $p(W; \discount)$ is the prior on $W$ represented by the density function in \cref{eq:densityw}.
The (non random) discount parameter $\discount$ is inferred by corresponding gradient ascent updates maximizing the likelihood of the model, which is described in \cref{sec:discount_inference}.

We specify a mean field variational distribution $q(W ; \theta) = \prod_{i=1}^n q(W_i ; \theta_i)$.
We considered several approximations for the marginals $q(W_i ; \theta_i)$ including truncated BFRY and truncated gamma distributions, however, in our experiments we found that the following \emph{rectified gamma distribution} performed well: 
\[
W_i &=_q \min\{ W_i', \ubound{n} \}
	,
	\\
	W_i' &\dist \gammadist( \theta_{i,\shp}, \theta_{i,\rte} )
	,
\]
independently for every $i\le n$, where $\theta_{i,\shp}$ and $\theta_{i,\rte}$ denote the shape and rate parameters of the gamma distribution, respectively, and the notation $=_q$ emphasizes that this formula holds under the variational distribution $q$.

\subsection{Stochastic gradient ascent}

We maximize the lower bound on the right hand side of \cref{eq:elbo1} by stochastic gradient ascent, where on the $t$-th step of the algorithm,
we make the following updates to  the parameters in parallel
\[
\theta_i^{(t+1)} \leftarrow \theta_i^{(t)}
	+ \rho_t \grad_{\theta_i} \EE_{q(W; \theta^{(t)})} [ \Lcal ( X, W; \theta^{(t)} ) ]
	,
	\label{eq:svi}
\]
for $i\le n$ and some sequence $(\rho_t)_{t\ge 1}$ of positive numbers satisfying the Robbins--Monro criterion \citep{robbins1951stochastic}
$
\sum\nolimits_t \rho_t = \infty 
$
and
$
\sum\nolimits_t \rho_t^2 < \infty
	,
$
and where
\[
\Lcal( X, W; \theta ) 
	&\defas \log p(X, W; \discount) - \log q(W; \theta)
	\\
	&\: =
	\sum_{(i,j) \in \edges} \log p( X_{i,j} \given W ) 
		 + \sum_{i=1}^n \log p( W_i ; \discount ) 
		 \nonumber
		 \\
		 &\qquad \qquad
		 - \sum_{i=1}^n \log q( W_i ; \theta_i ) 
	,
	\label{eq:logpq}
\]
where $\edges$ denotes the observed edges (both links and non-links) in the dataset.
We cannot evaluate the expectation (with respect to the rectified gamma distributions $q(W ; \theta)$) analytically, 
%
and so we elect to use a particular Monte Carlo approximation of this gradient detailed by \citet{knowles2015stochastic}, which was developed for gamma variational distributions and easily applies to the rectified gamma case.

Briefly, for every $i\le n$, create the collection of $S$ Monte Carlo samples from the variational distribution as follows: Independently for $s\le S$, let ${z_i^{(s)} \dist \uniform(0,1)}$, and set $W_i^{(s)} = \psi( z_i^{(s)}; \theta_i )$, where $\psi( z; \theta ) \defas \min\{ \igammacdf_\theta(z), \ubound{n} \}$ and $\igammacdf_\theta(x)$ is the inverse of the cumulative distribution function for a gamma random variable.  For convenience, we recall that
\[
\gammacdf_{a,b}(x)	
	= \int_0^x \frac{b^a}{\Gamma(a)}
		t^{a-1} e^{-b t} \dee t
		.
\]
For every $k\le n$, the gradient with respect to the parameters $\theta_k$ is then approximated by
\[
\begin{split}
&\grad_{\theta_k} \EE_{q(W; \theta)} [ \Lcal ( X, W; \theta ) ]
	\\
	&\qquad \approx
	\frac 1 {S} \sum_s 
	\grad_{W_k} \Lcal ( X, W^{(s)}; \theta ) \grad_{\theta_k} \psi (z_k^{(s)} ; \theta_k)
	,
	\label{eq:grad_approx1}
\end{split}
\]
where $W^{(s)} \defas (W_1^{(s)}, \dotsc, W_n^{(s)})$.
This estimator is unbiased and has low enough variance that often a single sample suffices for the approximation \citep{salimans2013fixed,kingma2013auto}.
The gradient of $\psi$ is nonzero only when ${\{\igammacdf_{\theta_k}(z_k^{(s)}) < \ubound{n}\}}$, in which case we may immediately obtain the partial derivative with respect to the rate parameter; in particular, we have
\[
\grad_{\theta_{k,\rte}} \psi (z_k^{(s)} ; \theta_k) 
	=
	\begin{cases}
	\frac{z_k^{(s)}}{\theta_{k,\rte} },
		&\text{if } \igammacdf_{\theta_k}(z_k^{(s)}) < \ubound{n}
		,
		\\
		0 , &\text{otherwise}
	.
	\end{cases}
\]
The partial derivative with respect to the shape parameter $\grad_{\theta_{k,\shp}} \psi (z_k^{(s)} ; \theta_k)$ does not have a closed form solution and must be approximated.
Different approximation routines are suggested by \citet{knowles2015stochastic} for different regimes of the shape parameter $\theta_{k,\shp}$, and we found these approximations to be accurate and efficient in our experiments.
%

\subsection{Minibatches of edges in the graph}

Computing the $n$ required gradients in \cref{eq:svi} may be done in parallel, and this computation, whether performed analytically or with automatic differentiation methods, scales with the number of edges in the graph.
This can be prohibitive for many network datasets, and we therefore introduce a further approximation where this gradient is evaluated on subsets (a.k.a. \emph{minibatches}) of the dataset, a technique from stochastic gradient ascent \citep{bottou2010large} adopted in the context of variational Bayesian inference by \citet{hoffman2013stochastic}.
In the case of a network dataset, we may select minibatches that are subsets of the observed edges in the graph.
In particular, write the gradient of \cref{eq:logpq} with respect to the variable $W_k$ (which is required by \cref{eq:grad_approx1}) as
\[
\grad_{W_k} \Lcal ( W^{(s)}; \theta )
	=
	\sum_{(i,j) \in \edges} g_{(i,j)} ( X, W^{(s)} ; k )
	,
\]
where $g_{(i,j)} (X,W ; k) \defas \grad_{W_k} [\log p( X_{i,j} \given W ) + \card{\edges}^{-1} \log p(W; \discount) - \card{\edges}^{-1} \log q(W; \theta)]$ is the gradient that ignores all but one edge of the graph.
%
We may therefore compute the unbiased estimate of this gradient
\[
\grad_{W_k} \Lcal ( W^{(s)}; \theta )
	\approx
	\frac{\vert \edges \vert}{\vert \minibatch \vert}
	\sum_{(i,j) \in \minibatch} g_{(i,j)} ( X, W^{(s)} ; k )
	,
\]
on a minibatch $\minibatch \subseteq \edges$ of the observed edges.
%

\subsection{Inference on the parameters $\discount$ and $\beta$}
\label{sec:discount_inference}

Without good prior knowledge of how to set the discount parameter $\discount$ and the sparsity parameter $\beta$ controlling the power law and sparsity behaviors of the graph, respectively, we infer their values from the data.
First consider the discount parameter, which we infer with gradient ascent. 
After every update to the latent variables $W$, we fix them to their mean under the distribution $q$, i.e., $\hat W \defas (\hat W_1, \dotsc, \hat W_n)$ where $\hat W_i = \EE_{q(W_i; \theta_i)} [ W_i ]$, and take a step in the direction of the gradient
\[
\grad_{\discount} \log p( \hat W ; \discount )
	&=
	\sum_{i=1}^n
	\grad_{\discount} \log p( \hat W_i ; \discount )
	\\
	&=
	\sum_{i=1}^n \Bigl [
		- \frac{\grad_{\discount} Z_{\discount,\beta} }{Z_{\discount,\beta}}
		- \log(\hat W_i)
		\Bigr ]
	,
	\label{eq:discountgrad}
\]
which is straightforward to derive from the density function in \cref{eq:densityw}, and 
where the normalization term
\[
Z_{\discount,\beta}
	\defas \int_0^{\ubound{n}} w^{-\discount-1}
		(1-e^{-w})
		\: \dee w
	\label{eq:Wnormconst}
\]
is a function of $\discount$ and $\beta$, if we let $\ubound{n} = n^\beta$ as suggested in \cref{sec:sparsity}.
We do not have a closed form solution for this term when $\ubound{n} < \infty$, and, unfortunately, inference on model parameters where the likelihood is difficult to evaluate is a challenging problem; for example, see the approaches taken by \citet{murray2006mcmc} on such problems, which those authors call \emph{doubly intractable distributions}.  
Accurate inference for $\discount$ is important in our model, because it controls the power law behavior of the graph.
In our experiments, we approximate the gradient in \cref{eq:discountgrad} for (fixed $\beta$) by approximating $Z_{\discount,\beta}$ (via \cref{eq:Wnormconst}) and
$
\grad_{\discount} Z_{\discount,\beta}
	=
	\int_0^{\ubound{n}}
	- w^{-\discount-1} (1-e^{-w})  \log w\: \dee w
	,	
$
with line integrals.
In the \cref{sec:experiments}, we demonstrate that this approximation works well in various regimes of $\discount$, with slight overestimation for moderate values.

Similar approaches to infer $\beta$ may be derived with finite difference approximations; 
we did not find these approaches successful in our experiments, however, and so we instead select $\beta$ by cross validation.

\section{Experiments}
\label{sec:experiments}

\begin{figure}[t!]
\centering
\includegraphics[scale=0.4]{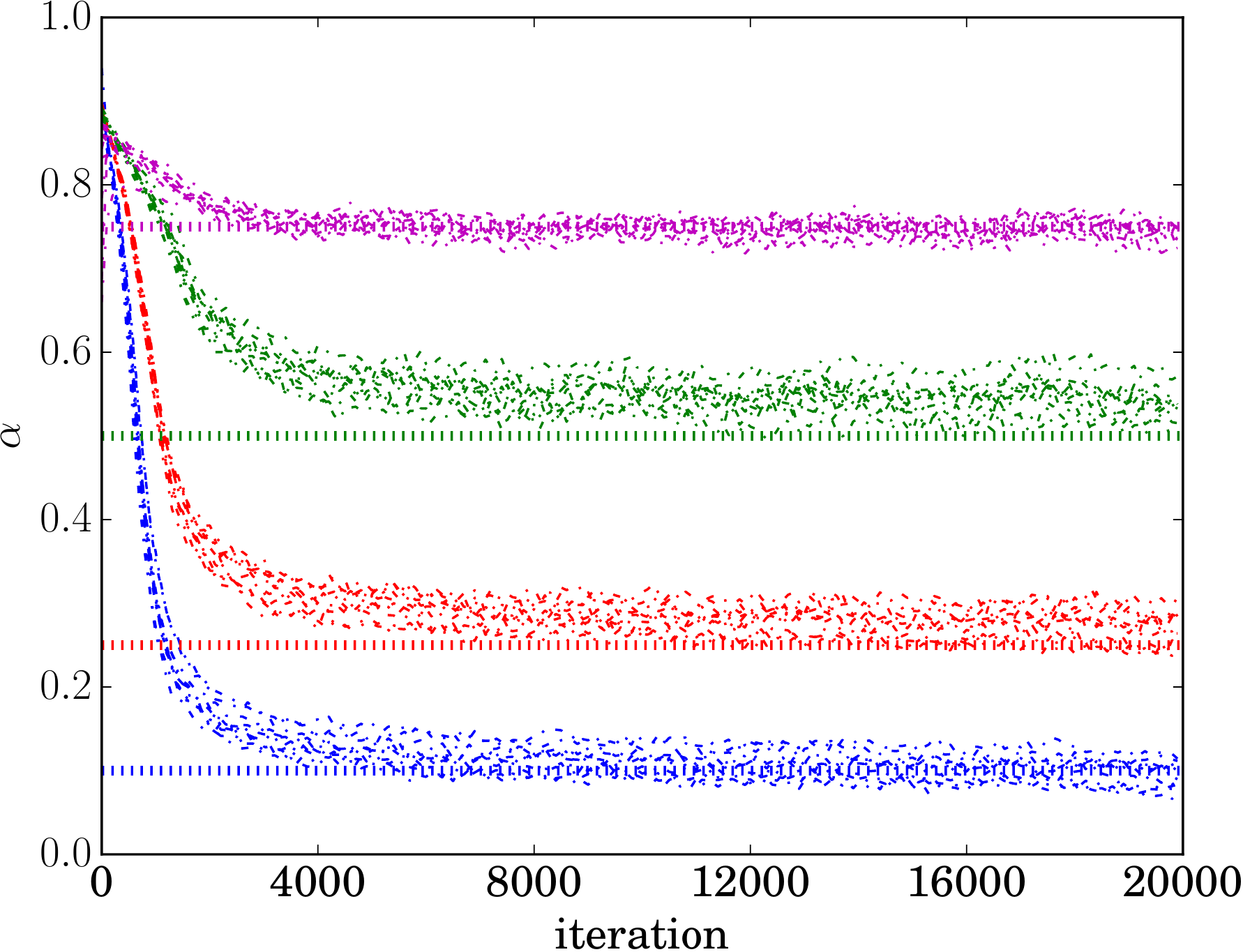}
\caption{Trace plots of the discount parameter $\discount$ during 10 different inference runs, each time simulating a dataset from the model with either $\discount \in \{0.1,0.3,0.5,0.7\}$ and intializing $\discount$ randomly.}
\label{fig:alpha_traces}
\end{figure}

\begin{table}[t]
\small
\caption{Comparision between the BFRY model and the Gamma baseline model when $\discount$ is known.
The test log-likelihoods were averaged over the last 4,000 of 20,000 gradient descent updates.}
\centering
\vskip 1em
\begin{tabular}{@{}M{1cm}M{0.9cm}M{2.575cm}M{2.5cm}@{}}
\toprule 
true $\discount$ & model & max test log-likel & avg test log-likel \\
\midrule
\multirow{2}{*}{$\discount=0.3$} &
BFRY & {\bf -57323.19 $\pm$ 91.62} & {\bf -57675.72 $\pm$ 31.71} \\
& Gamma & -71341.90 $\pm$ 116.82 & -71841.66 $\pm$ 47.38 \\
\addlinespace[0.1cm]
\multirow{2}{*}{$\discount=0.5$} &
BFRY & {\bf -21077.62 $\pm$ 79.64} & {\bf -21289.75 $\pm$ 34.23} \\
& Gamma & -24430.38 $\pm$ 73.06 & -24701.06 $\pm$ 11.31 \\
 \addlinespace[0.1cm]
 \multirow{2}{*}{$\discount=0.7$} &
BFRY & {\bf -7894.67 $\pm$ 41.84} & {\bf -8027.42 $\pm$ 51.08} \\
& Gamma & -8511.48 $\pm$ 22.45 & -8601.50 $\pm$ 15.42 \\
\bottomrule
\end{tabular}
\label{table:synthetic_summary}
\end{table}

\begin{table}[t]
\small
\caption{Comparision between the BFRY model and the Gamma baseline model on the air traffic, blogs, and social network datasets.
The test log-likelihoods were averaged over the last 4,000 of 20,000 gradient descent updates.}
\centering
\vskip 1em
\begin{tabular}{@{}M{0.9cm}M{0.9cm}M{2.575cm}M{2.575cm}@{}}
\toprule 
dataset & model & max test log-likel & avg test log-likel \\
\midrule
\multirow{2}{*}{500Air} &
BFRY & {\bf -1628.51 $\pm$ 10.46} & {\bf -1654.20 $\pm$ 6.79} \\
& Gamma & -1842.10 $\pm$ 3.97 & -1870.35 $\pm$ 0.28 \\
\addlinespace[0.1cm]
\multirow{2}{*}{polblogs} &
BFRY & {\bf -474.67 $\pm$ 32.20} & {\bf -503.20 $\pm$ 37.85} \\
& Gamma & -555.24 $\pm$ 18.27 & -596.78 $\pm$ 0.78 \\
 \addlinespace[0.1cm]
 \multirow{2}{*}{Fb107} &
BFRY & {\bf -18098.38 $\pm$ 20.50} & {\bf -18209.94 $\pm$ 12.86} \\
& Gamma & -18403.66 $\pm$ 31.76 & -18568.05 $\pm$ 2.79 \\
\addlinespace[0.1cm]
\multirow{2}{*}{openfl} &
BFRY & {\bf -16561.13 $\pm$ 137.89} & {\bf -16947.70 $\pm$ 177.21} \\
& Gamma & -17475.52 $\pm$ 31.97 & -17746.79 $\pm$ 6.65 \\
\bottomrule
\end{tabular}
\label{table:real_summary}
\end{table}

\newcommand{\colspace}{1.745cm}

\begin{table*}[t]
\small
\caption{Inferred hyperparameters in the experiments.}
\centering
\vskip 1em
\begin{tabular}{@{}p{2cm}M{\colspace}M{\colspace}M{\colspace}M{\colspace}M{\colspace}M{\colspace}M{\colspace}@{}}
\toprule 
& true $\alpha=0.3$ & true $\alpha=0.5$ & true $\alpha=0.7$ & 500Air & polblogs & Fb107 & openfl \\
\midrule
BFRY -- $\discount$ & 0.33 $\pm$ 0.00 & 0.53 $\pm$ 0.00 & 0.68 $\pm$ 0.00 & 0.23 $\pm$ 0.03 & 0.64 $\pm$ 0.06 & 0.00 $\pm$ 0.00 & 0.67 $\pm$ 0.21 \\
Gamma -- $\theta$ & 5.29 $\pm$ 0.01 & 1.42 $\pm$ 0.00 & 0.51 $\pm$ 0.00 & 5.10 $\pm$ 0.01 & 0.66 $\pm$ 0.00 & 33.58 $\pm$ 0.01 & 0.47 $\pm$ 0.00 \\
BFRY -- $\beta$ & -- & -- & -- & 1.08 $\pm$ 0.16 & 1.40 $\pm$ 0.00 & 0.80 $\pm$ 0.0 & 1.28 $\pm$ 0.10 \\
\bottomrule
\end{tabular}
\label{table:hyperparameters}
\end{table*}


We first demonstrate how the inference procedure in \cref{sec:discount_inference} can correctly differentiate between various regimes of $\discount$.  We ran an experiment where for each value $\discount \in \{0.1,0.3,0.5,0.7\}$, we simulated 10 datasets from the model with $n=1,000$ nodes, while fixing $\beta=1.0$.  For each simulated dataset, we ran an instance of the inference routine with $\discount$ randomly initialized.
In \cref{fig:alpha_traces}, we show the trace plots of alpha during each instance of the inference routine.
For comparison, the true values of $\discount$ are also shown as horizontal dashed lines.
We can see that the inference routine can correctly distinguish between these different regimes of $\discount$, with slight overestimation in the moderate $\alpha$ regime.
Interestingly, despite random initializations of $\discount \in (0,1)$, the algorithm always immediately inflates $\discount$ to around 0.9, and then slowly decreases this value during inference, regardless of what value of $\discount$ generated the data.

We next demonstrate that accurately capturing power law structures in datasets will improve predictive performance. 
While fixing $\beta=1.0$, we simulate three network datasets with 5,000 nodes from our model with discount parameters $\alpha = 0.3, 0.5,$ and $0.7$, respectively, which therefore exhibit increasingly lighter-tailed degree distributions. The generated graphs have 117,300, 32,925, and 9,460 links, respectively. 
To establish a baseline model that does not exhibit power law degree distributions but is otherwise comparable to our model, we implement the generalized random graph where the node-specific weights are constructed from the gamma random variables $W_i \dist \gammadist(\theta, 1)$, for some positive parameter $\theta$, i.i.d.\ for every node $i\le n$.
Note that the parameter $\theta$ controls the sparsity of the generated graph; larger values of $\theta$ imply denser graphs.  
It follows analogously to \cref{result:grg_scaled_power_law} that 
\[
\Pr\{\degree{i}=k\} \sim \frac{k^{\theta-1}}{2^{k+\theta}},
\]
for $k \gg 1$, as $n\to\infty$. 
This model therefore does not exhibit power law behavior, as desired.
We refer to this model as ``Gamma'' and the power law graph model as ``BFRY''.

We ran an experiment holding out 20\% of the edges in the simulated graphs as test sets, training the two models on the remaining 80\% of the edges.  We used a mini-batch size of 5,000 edges (note that the training dataset corresponds to almost 10 million observed edges).  We ran each inference procedure for 20,000 steps of stochastic gradient ascent updates, using Adam \citep{kingma2015adam} to adjust the learning rates at each step.  We repeated each experiment 5 times, each time holding out a different test set and using a different random initialization. Again, for this experiment we fixed $\beta=1$.
In \cref{table:synthetic_summary} we report a mean log-likelihood metric for the test datasets, where the metric for each run is obtained by averaging the test log-likelihoods across the states for the last 4,000 steps of the inference procedure; the displayed intervals are at $\pm 1$ standard deviation about the metric, from across the 5 repeats.  We also report a max log-likelihood metric, which simply records the maximum test log-likelihood across the last 4,000 steps of the inference procedure, instead of the average.
The best performing method is highlighted in bold (which in each case was the BFRY model).

In each case, we see that the BFRY model achieves higher test log-likelihood metrics than the Gamma model, as expected, implying that accurately capturing a power law degree distribution improves predictive performance (when power law behavior is truly present in the network).
In \cref{table:hyperparameters}, we report the inferred values for $\discount$, which were reasonably accurate, though we see slight overestimation for some regimes, as seen in the demonstration earlier.
For the baseline Gamma model, we optimized the hyperparameter $\theta$ using gradient ascent maximizing the evidence lower bound of the model (c.f.~\cref{eq:elbo1}), and the inferred values are also reported in \cref{table:hyperparameters}.

Next, we ran similar experiments on the following network datasets, each of which are expected to exhibit power law degree distributions:
\begin{itemize}[nolistsep]
\item `USTop500Airports': 500 nodes, 2,980 links

\item `openflights': 7,976 nodes, 15,243 links

\item `polblogs': 1,490 nodes, 9,517 links

\item `Facebook107': 1,034 nodes, 26,749 links

\end{itemize}
Where appropriate, we saved only the upper triangular parts of the adjacency matrices.
The `USTop500Airports' dataset contains the (undirected, unweighted) flight connections between the 500 busiest US airports. 
The similar, though much larger, `openflights' dataset contains the flight connections between non-US airports.
Scale-free networks have been proposed for such \emph{traffic networks}, detailed for these datasets by \citet{colizza2007reaction}.
The `polblogs' dataset contains the links between political blogs (judged by hyperlinks between the front webpages of the blogs) in the period leading up to the 2004 US presidential election, which is observed to exhibit power law degree distributions by \citet{adamic2005political}.
The `Facebook107' dataset contains ``friendships'' between users of a Facebook app, collected by \citet{leskovec2012learning}; social networks are widely studied for their power law degree distributions.

For both the Gamma and BFRY models, we ran our variational inference procedure for 20,000 steps on each dataset. 
As before, we repeated the experiment 5 times for each network, each time holding out a different 20\% of the edges in the network as a testing set.
We selected the value of $\beta$ from among the grid $\{0.6, 0.9, 1.0, 1.2, 1.4\}$ with 5-fold cross validation on the training set.
We set the minibatch size to be equal to the number of nodes in the graph; for example, we used minibatches of 1,490 edges for the polblog dataset.
The evaluation metrics on the test datasets are summarized in \cref{table:real_summary}, and the inferred hyperparameter values are reported in \cref{table:hyperparameters}.
We see that the BFRY model once again outperforms the Gamma baseline model, according to the test log-likelihood metrics.

Probabilistic inference on $\discount$ by the BFRY model provides some of the most interesting analyses here.
With $\discount \approx 0.00$ (underflowing our machine's precision), the Facebook107 social network has the degree distribution with the heaviest tails, followed by the USTop500Airports traffic network with $\discount \approx 0.23$, the polblog citation network with $\discount \approx 0.64$, 
and the openflights network has the lightest tailed degree distribution with $\discount \approx 0.67$.
%


\section{Future work}


Future work could focus on implementing the latent factor modeling generalizations presented in \cref{sec:blockmodels}, which are natural assumptions in many domains where networks are expected to exhibit power law degree distributions.
Alternative approaches to inference on the sparsity parameter $\beta$ should also be explored, since controlling the sparsity in the graph was important for good predictive performance.

\newpage

\section*{Acknowledgements}

The authors thank Remco van der Hofstad for helpful advice and anonymous reviewers for helpful feedback.
J.\ Lee and S.\ Choi were partly supported by an Institute for Information \& Communications Technology Promotion (IITP) grant, funded by the Korean
government (MSIP) (No.2014-0-00147, Basic Software Research in
Human-level Lifelong Machine Learning (Machine Learning Center)) and Naver, Inc. 
C.\ Heaukulani undertook this work in part while a visiting researcher at the Hong Kong University of Science and Technology, who along with L.\ F.\ James was funded by grant rgc-hkust 601712 of the Hong Kong Special Administrative Region.

\bibliography{random_graph}
\bibliographystyle{icml2017} 


\onecolumn

\icmltitle{Supplementary Material: Bayesian inference on random simple graphs\\
		with power law degree distributions}




\begin{icmlauthorlist}
\icmlauthor{Juho Lee}{pos}
\icmlauthor{Creighton Heaukulani}{cam}
\icmlauthor{Zoubin Ghahramani}{cam,ub}
\icmlauthor{Lancelot F. James}{ust}
\icmlauthor{Seungjin Choi}{pos}
\end{icmlauthorlist}

\icmlaffiliation{pos}{Pohang University of Science and Technology (POSTECH), Pohang, South Korea}
\icmlaffiliation{cam}{University of Cambridge, Cambridge, United Kingdom}
\icmlaffiliation{ub}{Uber AI Labs, San Francisco, CA, USA}
\icmlaffiliation{ust}{Hong Kong University of Science and Technology (HKUST), Clearwater Bay, Hong Kong}

\icmlcorrespondingauthor{Juho Lee}{stonecold@postech.ac.kr}

\icmlkeywords{relational data, network models, scale-free random graphs, variational inference}

\vskip 0.3in



\printAffiliationsAndNotice{}  

\section{Proofs}

We prove \cref{result:main} and \cref{result:grg_scaled_power_law} in the paper.
First consider the following redefinition of our model with slightly different notation; let $W_n$ be a random variable constrained on $(0, \ubound{n}]$, with density
\[
f_n(\dee w) = \frac{1}{Z_n} w^{-\discount-1}(1-e^{-w}) \Indicator{0<w\leq C_n}  \dee w,
\label{eq:W_density}
\]
where $C_1, C_2, \dots, $ is a sequence of positive numbers satisfying
\[
\lim_{n\to\infty} \ubound{n} = \infty, \qquad \lim_{n\to\infty} \ubound{n}^\discount / n= 0.
\label{eq:C_cond}
\]
Note that $Z_n \to \Gamma(1-\discount)/\discount$ as $n\to\infty$, and so the sequence of densities $f_n(dw)$ converges pointwise to the
density of the BFRY distribution
\[
f(w) = \frac{\discount}{\Gamma(1-\discount)}w^{-\discount-1}(1-e^{-w})\Indicator{w > 0},
\]
and $W_n$ converges in distribution to a BFRY random variable. 
Let $W_{n,1}, \dots, W_{n,n}$ be $n$ i.i.d.\ copies of $W_n$. A random simple graph $X$ is then defined to be a collection of Bernoulli random variables as follows:
\[
\label{eq:grgmodel}
\Pr\{ X_{ij}=1 \given r_{i,j}\} = \frac{r_{i,j}}{1 + r_{i,j}}
	,
	\quad r_{i,j} = U_i U_j
	, 
	\quad U_i = \frac{W_{n,i}}{\sqrt{L_n}}
	,
\]
where $L_n \defas \sum_{i=1}^n W_{n,i}$. 
We will write $X \given r \sim \GRGLAW(n,r)$, where $r \defas (r_{i,j} \colon i < j \le n)$.   


We begin with a sequence of Lemmas.
Define a sequence of random variables $V_{s, n}$, for every $s, n \ge 1$, by
\[
V_{s,n} \defas \frac{W_n}{\ubound{n}^{s-\discount}}.
\]
Let $V_{s,n,1}, \dots, V_{s,n,n}$ be $n$ i.i.d.\ copies of $V_{s,n}$, and denote the empirical mean of these copies by
\[
\bar V_{s,n}
	\defas \frac 1 n
	\sum_{i=1}^n
	V_{s,n,i}
	.
	\label{eq:Vmeandef}
\]
The expectation of $V_{s,n}$ is finite for all $s, n < \infty$, and is computed as
\[
\EE[V_{s,n}] 
	&= \frac{1}{Z_n \ubound{n}^{s-\discount}} 
		\int_0^{\ubound{n}} w^{s-\discount-1}(1-e^{-w}) \dee w 
		\nonumber 
		\\
	&= \frac{1}{Z_n}\bigg\{ 
		\frac{1}{s-\discount} 
		- \frac{\gamma(s-\discount, \ubound{n})}{\ubound{n}^{s-\discount}}
		\bigg\}
	,
\]
where $\gamma(\cdot, \cdot)$ is the lower incomplete gamma function.

Let $\pto$ denote convergence in probability.  The following lemma is a standard mean convergence result:
%
\begin{lem}
$\bar V_{s,n} \pto \EE[V_{s,n}]$, as $n\to \infty$. 
\label{lem:wlln}
\end{lem}

\begin{proof}
For all $\varepsilon > 0$, by Chebyshev's inequality and the condition in \cref{eq:C_cond},
\[
\Pr\{ \vert \bar V_{s,n} - \EE[V_{s,n}] \vert \ge \varepsilon\} &\leq \frac{\mathrm{Var}(V_{s,n})}{n \varepsilon^2} 
\leq \frac{\EE[V_{s,n}^2]}{n \varepsilon^2} 
= \frac{1}{Z_n \varepsilon^2} \bigg\{ \frac{\ubound{n}^\discount}{n(2s-\discount)}
- \frac{\gamma(2s-\discount, \ubound{n})}{n \ubound{n}^{2s-2\discount}}\bigg\} \to 0,
\]
as $n \to \infty$, as desired.
\end{proof}

The following lemma will be used to study various higher order moments in later results:
%
\begin{lem}
For $s\geq 2$, 
\[
\ratio_{s,n} \defas 
	\frac{\sum_{i=1}^n W_{n,i}^s}{(\sum_{i=1}^n W_{n,i})^s} 
	\pto 0
	,
	\qquad
	\text{as }
	n\to \infty
	.
\]
\label{lem:R}
\end{lem}

\begin{proof}
We have
\[
\ratio_{s,n} = \frac{ n \ubound{n}^{s-\discount} \bar V_{s,n}}{ n^s \ubound{n}^{s-s\discount} \bar V_{1,n}^s} = 
\bigg(\frac{\ubound{n}^\discount}{n}\bigg)^{s-1} \frac{\bar V_{s,n}}{\bar V_{1,n}^s}.
\]
As $n\to \infty$, the first factor on the right hand side clearly converges to zero (c.f.~\cref{eq:C_cond}), and, by \cref{lem:wlln}, the second term
converges to a constant in probability. 
\end{proof}

Recall that $\degree{i} \defas \sum_{j\ne i} X_{i,j}$ is the degree of the $i$-th node in the graph $X \given r \dist \GRGLAW(n,r)$, given by \cref{eq:grgmodel}.
The following result will show up in later calculations involving the probability generating function (PGF) of the degree random variables $\degree{i}$:
%
\begin{lem}
For every collection $t_1, \dotsc, t_n$ with $\vert t_i \vert \leq 1$, for $i\le n$,
\[
\EE \Bigl [ 
	\prod_{i=1}^n t_i^{\degree{i}} \given W_{n,1} = w_{1} ,\dotsc, W_{n,n} = w_{n}
	\Bigr ]
	= \prod_{i < j \leq n } 
		\frac{L_n + t_i t_j w_i w_j}{L_n + w_i w_j},
\]
for positive $w_1, \dotsc, w_n$.
\label{lem:pgf_degree}
\end{lem}

\begin{proof}
The proof is given by \citet{britton2006generating}.
\end{proof}

The following result studies a representation of the PGF of the degree random variables and their higher order moments:
%
\begin{lem}
Fix a node $k\le n $. Define 
\[
\pgf{n,k} (t; w_k) \defas \prod_{i\neq k} \frac{L_{n,-k} + w_k + t w_k W_{n,i}}{L_{n,-k} + w_k + w_k W_{n,i}}
	,
	\qquad 
	\text{for }
	\vert t \vert \leq 1
	,
	\text{ and }
	w_k > 0
	,
\]
where $L_{n,-k} \defas \sum_{i\neq k} W_{n,i}$.
Note that the $s$-th derivative $\mpgf{n,k}{s}(t; w_k)$ exists for all $s \geq 0$. 
For all $s\geq 0$, the following hold:
\begin{enumerate}[nolistsep]
\item $\mpgf{n,k}{s}(t; w_k)$ is uniformly bounded, for all $n\ge 1$; 
\item $\mpgf{n,k}{s}(t; w_k) \pto w_k^{s} \exp\{(t-1) w_k\}$, as $n\to \infty$.
\end{enumerate}
\label{lem:cond_converge}
\end{lem}

\begin{proof}
In the case $s=0$, $\pgf{n,k}(t; w_k)$ is trivially bounded by 1 since $\vert t \vert\leq 1$. By the Taylor series expansion $\log(1 + x) = x + O(x^2)$, we have
\[
\pgf{n,k}(t; w_k) = \exp\bigg\{ (t-1)w_k \frac{L_{n,-k}}{L_{n,-k} + w_k} 
		+ O\bigg(
			w_k^2 \frac{\sum_{i\neq k} W_{n,i}^2}{(L_{n,-k} + w_k)^2}
		\bigg)
		\bigg\}.
\]
By \cref{lem:wlln},
\[
\frac{L_{n,-k}}{L_{n,-k} + w_k} 
	= \frac{\bar V_{1,n,-k}}
		{\bar V_{1,n,-k} + w_k/(n-1)/\ubound{n}^{1-\discount}} 
 	\pto 1
	,
\]
where $\bar V_{1,n,-k}$ is the empirical mean in \cref{eq:Vmeandef} excluding the element $V_{1,n,k}$. Furthermore, by \cref{lem:R},
\[
O\bigg( w_k^2 \frac{\sum_{i\neq k} W_{n,i}^2}{(L_{n,-k} + w_k)^2}
	\bigg) 
	\leq
	O( w_k^2 \ratio_{2,n,-k}) 
	\pto 0
	,
\]
where $\ratio_{s,n,-k}$ is $\ratio_{s,n}$ computed without $V_{s,n,k}$. 
Combining, we have 
\[
\pgf{n,k}(t; w_k) \pto \exp\{(t-1)w_k\}.
\label{eq:F_0}
\]

Before proceeding for $s\geq 1$, we define
\[
\ingf{r,n,k}(t;w_k) \defas \sum_{i\neq k} \frac{W_{n,i}^r}{(L_{n,-k} + w_k + t w_k W_{n,i})^r},
\]
for all $r,n \geq 1$. One can easily see that $\ingf{r,n,k}(t;w_k)\leq 1$ for all $r,n\geq 1$. For $r=1$, we have
\[
\sum_{i\neq k} \frac{W_{n,i}}{L_{n,-k} + w_k + t w_k C_n} \leq \ingf{1,n,k}(t;w_k) \leq 1,
\]
and
\[
\sum_{i\neq k} \frac{W_{n,i}}{L_{n,-k} + w_k + t w_k C_n} &=
\frac{1}{1 + w_k/L_{n,-k} + t w_k C_n/L_{n,-k}} \nonumber\\
&= \bigg\{ 1 + \frac{w_k}{(n-1)C_n^{1-\alpha}} \bar V_{s,n,-k}^{-1} 
+ t w_k \frac{C_n^\alpha}{n}\frac{n}{n-1} \bar V^{-1}_{s,n,-k}\bigg\}^{-1} \pto 1.
\label{eq:Q_1}
\]
Hence, by the squeeze theorem, $\ingf{1,n,k}(t;w_k) \pto 1$. For $r \geq 2$, we have
\[
0 \leq \ingf{r,n,k}(t;w_k) \leq M_{r,n,-k} \pto 0,
\label{eq:Q_2}
\]
by \cref{lem:R}. Hence, we have $\ingf{r,n,k}(t;w_k)\pto 0$ for $r \geq 2$.

Now we show that
\[
\mpgf{n,k}{s}(t;w_k) = w_k \mpgf{n,k}{s-1}(t;w_k) \ingf{1,n,k}(t;w_k) + \sum_{r=2}^s a_{s,r} \mpgf{n,k}{s-r}(t;w_k) \ingf{r,n,k}(t;w_k),
\label{eq:F_derv}
\]
for some constants $\{a_{s,r}\}$ for all $s\geq 1$ and $r \geq 2$. We proceed by the mathematical induction. For $s=1$, 
\[
\mpgf{n,k}{1}(t;w_k) &= \sum_{i\neq k} \frac{w_k W_{n,i}}{L_{n,-k} + w_k + w_k W_{n,i}}
\prod_{j\neq i,k} \frac{L_{n,-k} + w_k + t w_k W_{n,j}}{L_{n,-k} + w_k + w_k W_{n,j}} \nonumber\\
&= w_k F_{n,k}(t;w_k) \ingf{1,n,k}(t;w_k).
\]
Now by the inductive hypothesis, 
\[
\mpgf{n,k}{s+1}(t;w_k) &= w_k \mpgf{n,k}{s}(t;w_k) \ingf{1,n,k}(t;w_k) - w_k^2 \mpgf{n,k}{s-1}(t;w_k) \ingf{2,n,k}(t;w_k) \nonumber\\
& + \sum_{r=2}^s a_{s,r} (\mpgf{n,k}{s+1-r}(t;w_k) \ingf{r,n,k}(t;w_k) - r w_k \mpgf{n,k}{s-r} \ingf{r+1,n,k}(t;w_k)) \nonumber\\
&= w_k \mpgf{n,k}{s}(t;w_k) \ingf{1,n,k}(t;w_k) + \sum_{r=2}^{s+1} a_{s+1,r} \mpgf{n,k}{s+1-r}(t;w_k) \ingf{r,n,k}(t;w_k),
\]
where
\[
a_{s+1,2} = a_{s,2} - w_k^2 , \quad a_{s+1,r} = a_{s,r} - a_{s,r-1} (r-1)  w_k \quad \textrm{ for } r \geq 2,
\]
so the inductive argument holds. 

Having \eqref{eq:F_derv}, by mathematical induction, we can easily show that $\mpgf{n,k}{s}(t;w_k)$ is uniformly bounded for all $s,n\geq 1$.
Moreover,
\[
\mpgf{n,k}{1}(t;w_k) = w_k F_{n,k}(t;w_k) \ingf{1,n,k}(t;w_k) \pto w_k \exp\{(t-1)w_k\},
\]
by \eqref{eq:F_0} and \eqref{eq:Q_1}. Combining this with \eqref{eq:Q_2}, by mathematical induction, we can show that for all $s \geq 1$,
\[
\mpgf{n,k}{s}(t;w_k) \pto w_k^s \exp\{(t-1)w_k\}.
\]

\ifx
In the case $s=1$, we have
\[
\mpgf{n,k}{1}(t; w_k)
	&= \sum_{i\neq k} 
		\frac{w_k W_{n,i}}{L_{n,-k} + w_k + w_k W_{n,i}} 
		\prod_{j\neq i, k} \frac{L_{n,-k} + w_k + t w_k W_{n,j}}{L_{n,-k} + w_k + w_k W_{n,j}}
		\nonumber
		\\
	&= w_k \pgf{n,k}(t; w_k)
		\sum_{i\neq k} \frac{W_{n,i}}{L_{n,-k} + w_k + tw_k W_{n,i}}
	.
\]
Note that
\[
\mpgf{n,k}{1}(t; w_k) 
	\leq w_k \pgf{n,k}(t; w_k) 
	\sum_{i\neq k} \frac{W_{n,i}}{L_{n,-k}} 
	= w_k \pgf{n,k}(t; w_k)
	,
\]
and $\mpgf{n,k}{1}(t; w_k)$ is therefore uniformly bounded by $w_k$. Also note that
\[
\mpgf{n,k}{1}(t; w_k) 
	 &\geq w_k \pgf{n,k}(t; w_k)
	\sum_{i\neq k} 
	\frac{W_{n,i}}{L_{n,-k} + w_k + t w_k \ubound{n}}
	\nonumber 
	\\
	& = \frac{w_k \pgf{n,k}(t; w_k) L_{n,-k}}{L_{n,-k} + w_k + t w_k \ubound{n}}
	\nonumber
	\\
	& = \frac{w_k \pgf{n,k}(t; w_k)}{1 + \frac{w_k}{(n-1) \ubound{n}^{1-\discount}} 
		\bar V_{s,n,-k}^{-1} + t w_k \frac{\ubound{n}^\discount}{n}\frac{n}{n-1} \bar V_{s,n,-k}^{-1} } 
		\pto 
		w_k \exp\{(t-1)w_k\}
		.
	\label{eq:first_derivative}
\]
Then by the squeeze theorem, we have $\mpgf{n,k}{1}(t; w_k) \pto w_k \exp\{(t-1)w_k\}$.

Finally, for the case $s\geq 2$, we will first show by induction that
\[
\mpgf{n,k}{s}(t; w_k) 
	= w_k \mpgf{n,k}{s-1}(t; w_k) 
		 \ingf{1,n,k}(t; w_k) 
		+ \sum_{r=2}^s a_{s,r} w_k^r \mpgf{n,k}{s-r}(t; w_k) \ingf{r,n,k}(t; w_k)
	,
	\label{eq:recurse}
\]
for some constants $\{a_{s,r}\}_{r=2}^s$ and where
\[
\ingf{r,n,k}(t ; w_k) \defas 
	\sum_{i\neq k} \frac{W_{n,i}^r}{(L_{n,-k} + w_k + t w_k W_{n,i})^r}
	.
\]
The base case $s=2$ is trivial. 
By the inductive hypothesis,
\[
\pgf{n,k}{s+1}(t; w_k) &= w_k \mpgf{n,k}{s}(t; w_k) \ingf{1,n,k}(t; w_k) 
	- w_k^2 \mpgf{n,k}{s-1}(t; w_k) \ingf{2,n,k}(t; w_k)
	\\
	&\qquad \qquad
	 + \sum_{r=2}^s a_{s,r}( w_k^r  \ingf{r,n,k}(t; w_k) 
	 		- w_k^{r+1} \mpgf{n,k}{s-r}(t; w_k) \ingf{r+1,n,k}(t; w_k))
		,
\]
and so \cref{eq:recurse} holds.
\PROBLEM{I don't understand from here onwards.}
Since $\ingf{r,n,k}(t; w_k) \leq 1$ for all $r, n$, it is straightforward to show with another inductive argument that
$\mpgf{n,k}{s}(t; w_k)$ is uniformly bounded for all $s$ and $t$. 
Also note that, by \cref{lem:R},
\[
\ingf{r,n,k}(t; w_k) \leq \ratio_{r,n,-k} \pto 0,
\]
for all $r \geq 2$. The base case $\ingf{1,n,k}(t; w_k) \pto 1$ was shown in \eqref{eq:first_derivative}. Combining,
we have $\mpgf{n,k}{s}(t; w_k) \pto w_k^s\exp\{(t-1) w_k\}$.
\fi
\end{proof}

We will now use our collected results to analyze the asymptotic distribution of the degree random variables; the following result characterizes this distribution:
%
\begin{lem}
Fix a node $k$. Given $\{W_{n,k}=w_k\}$, for some $w_k>0$, the degree $\degree{k}$ of node $k$ converges in distribution to a Poisson random variable with rate $w_k$, as $n\to \infty$.
\label{lem:cond_dist}
\end{lem}

\begin{proof}
The PGF of $\degree{k}$ is given by
\[
\EE [ t^{\degree{k}} \given W_{n,k}=w_k] 
	= \EE[ \pgf{n,k}(t; w_k)]
	,
	\qquad
	\text{for }
	\vert t \vert \le 1
	.
	\label{eq:dpgf}
\]
Note that these expectations are under the $\sigma$-field generated by $\{ W_k = w_k \}$.
For all $s\geq 0$, we will derive the limit of $\Pr \{\degree{k} = s \given w_k\}$, as $n\to \infty$, which we note is given by the $s$-th order derivatives of the PGF in \cref{eq:dpgf}, evaluated at the argument $t = 0$.
It therefore suffices to show that $\EE [ \mpgf{n,k}{s}(t; w_k)] \to w_k^s \exp\{(t-1) w_k\}$, as $n\to \infty$, for all $s \geq 0$. 
By \cref{lem:cond_converge}, we know that $\mpgf{n,k}{s}(t; w_k)$ is uniformly bounded and that $\mpgf{n,k}{s}(t; w_k) \pto w_k^s \exp\{(t-1)w_k\}$, as $n\to \infty$.
Therefore, by uniform integrability, 
\[
\lim_{n\to\infty} \EE[ \mpgf{n,k}{s}(t; w_k) ] = \EE
	\Big [ 
		\lim_{n\to\infty} \mpgf{n,k}{s}(t; w_k)
		\Big ] 
		= w_k^s \exp\{(t-1)w_k\}
		.
\]
\end{proof}

We are now ready to prove the main theorems in the paper.
%

\begin{proof}[Proof of \cref{result:main}]
We will first verify that, for $y \gg 1$, $\Pr\{ \degree{k}=y\} \to c y^{-1-\discount}$ for every node $k$ and for some constant $c>0$ as $n\to\infty$. 
By \cref{lem:cond_dist}, conditioned on $\{W_k = w_k\}$, the degree $\degree{k}$ converges in distribution to a Poisson random variable with rate $w_k$.
Then by dominated convergence,
\[
\lim_{n\to\infty} \Pr\{ \degree{k} = y \} 
	&= \lim_{n\to\infty} \int_0^{\infty} \Pr\{D_k=y|w_k\} p_n(\dee w_k)
	\nonumber
	\\
	&= \int_0^\infty \frac{w_k^y e^{-w_k}}{y!} 
		p(\dee w_k) 
		\nonumber 
		\\
	&= \frac{\discount\Gamma(y-\discount)}{y!\Gamma(1-\discount)} (1-2^{\discount-y})
	.
\]
By the asymptotics of the Gamma function, for $y \gg 1$, we have
\[
\lim_{n\to\infty} \Pr\{\degree{k} = y\} = c y^{-1-\discount},
\]
for some constant $c$.

Next we show that, for any finite $m$, the collection of random variables $\degree{1}, \dotsc, \degree{m}$ are asymptotically independent, as $n\to \infty$.
We compute the (joint) probability generating function of $(\degree{1}, \dotsc, \degree{m})$, with $\vert t_i \vert\leq 1$ for $i=1,\dotsc, m$.
By \cref{lem:pgf_degree},
\[
\EE \Bigl [ \prod_{i=1}^m t_i^{\degree{i}} \Bigr ] 
	&=
	\EE\Bigl [ \prod_{i=1}^m \prod_{j=i+1}^m 
		\frac{L_n + t_it_j W_{n,i} W_{n,j}}{L_n + W_{n,i} W_{n,j}} 
		\prod_{j=m+1}^n 
			\frac{L_n + t_i W_{n,i} W_{n,j}}{L_n + W_{n,i} W_{n,j}}
		\Bigr ] 
		\nonumber
		\\
	&= \EE \Bigl [ 
			\EE \Bigl [ 
					\prod_{i=1}^m \prod_{j=i+1}^m 
					\frac{L_{n,m+1:n} + \ell_{n,1:m} + t_i t_j w_i W_{n,j}}
					{L_{n,m+1:n} + \ell_{n,1:m} + w_i W_{n,j}} 
		\nonumber \\
	&\qquad \qquad \times 
	\prod_{j=m+1}^n 
	\frac{L_{n,m+1:n} + \ell_{n,1:m} + t_i w_i W_{n,j}}{L_{n,m+1:n} + \ell_{n,1:m} + w_i W_{n,j}}
		\given W_{n,1:m} = w_{1:m} \Bigr ]
		\Bigr ]
		.
\]
Given $w_{1:m}$, by a similar argument as in the proof of \cref{lem:cond_converge}, one can easily show that
\[
\prod_{j=i+1}^m 
\frac{L_{n,m+1:n} + \ell_{n,1:m} + t_i t_j w_i W_{n,j}}{L_{n,m+1:n} + \ell_{n,1:m} + w_i W_{n,j}}
\pto 1,
	\quad
	\text{as } n\to \infty
	,
\]
and
\[
\prod_{j=m+1}^n  \frac{L_{n,m+1:n} + \ell_{n,1:m} + t_i w_i W_{n,j}}{L_{n,m+1:n} + \ell_{n,1:m} + w_i W_{n,j}}
	\pto \exp\{(t_i-1) w_i\}
	,
	\quad
	\text{as }
	n\to \infty
	.
\]
Hence, again by a similar argument as in the proof of \cref{lem:cond_converge}, we have
\[
\lim_{n\to\infty}\EE \Bigl [ 
	\prod_{i=1}^m t_i^{\degree{i}}
	\Bigr ] 
	 = \prod_{i=1}^m \EE[ \exp\{(t_i-1)W_i\} ]
	,
\]

that is, the joint PGF asymptotically factorizes into the product of the PGFs for i.i.d.\ random variables, and the result follows.
\end{proof}

\begin{proof}[Proof of \cref{result:sparsity}]
Using the fact that the expected number of nodes $E_n \defas \sum_{i=1}^n \degree{i}/2$, we may take $t_1= \dots = t_n = \sqrt{t}$ and obtain
\[
\EE [t^{E_n}] 
	= \EE \Bigl [ 
		\prod_{i<j\leq n} \frac{L_n + t W_{n,i} W_{n,j}}{L_n + W_{n,i} W_{n,j}}
	\Bigr ]
	.
\]
We evaluate the derivative of the PGF to obtain the first moment
\[
\EE [E_n] 
	&= \frac{\partial \EE[t^{E_n}]}{\partial t} \Bigr |_{t=1} 
	= \EE \Bigl [ \sum_{i<j\leq n} \frac{W_{n,i} W_{n,j}}{L_n + W_{n,i} W_{n,j}}
		\Bigr ] 
	\leq \frac{1}{2} \EE \Bigl [ 
		\sum_{i \leq j \leq n} \frac{W_{n,i} W_{n,j}}{L_n} 
		\Bigr ] 
		= \frac{n}{2} \EE[W_n]
		.
\]
Since 
\[
\EE[W_n] = \frac{1}{Z_n} 
	\Bigl \{ 
		\frac{\ubound{n}^{1-\discount}}{1-\discount} - \gamma(1-\discount, \ubound{n})
	\Bigr \}
	,
\]
we have
\[
\EE [E_n] = O(n \ubound{n}^{1-\discount})
	.
\]
\end{proof}


\begin{proof}[Proof of \cref{result:grg_scaled_power_law}]
Recall that
\[
\Pr\{X=x \given r\} 
	= \prod_{i < j \le n} 
	\frac{r_{i,j}}{1 + r_{i,j}} = G^{-1}(r) 
	\prod_{i < j \le n} A_{i,j}^{x_{i,j}} 
	\prod_{i=1}^n U_i^{\degree{i}}
	,
\]
where $A \defas (A_{i,j})_{i<j\le n}$ 
and
\[
G(r) \defas \prod_{i < j \le n} (1 + A_{i,j} U_i U_j)
	.
\]
Since $\sum_x \Pr\{X=x\given r\} = 1$, we have
\[
G(r) = \sum_x \prod_{i < j\le n} A_{i,j}^{x_{i,j}} 
	\prod_{i=1}^n u_i^{\degree{i}}.
\]
The joint PGF of $(\degree{1}, \dotsc, \degree{n})$ is then
\[
\EE \Bigl [ 
	\prod_{i=1}^n t_i ^{\degree{i}} \given A, W_{n,1:n}
	\Bigr ] 
	&= \sum_x \Pr\{X=x \given r\} \prod_{i=1}^n t_i^{\degree{i}(x)} 
	\nonumber
	\\
	&= G^{-1}(r) \sum_x \prod_{i<j\le n} A_{i,j}^{x_{i,j}} 
		\prod_{i=1}^n (t_i U_i)^{\degree{i}}
		\nonumber
		\\
	&= \prod_{i < j\le n} \frac{1 + A_{i,j} t_i t_j U_i U_j}{1 + A_{i,j} U_i U_j}
	.
\]
The remainder of the proof follows analogously to the proof of \cref{result:main} above.
\end{proof}


\end{document}